\documentclass{article}
\usepackage{amsmath}

\usepackage{amsmath}
\newtheorem{theorem}{Theorem}

\newtheorem{proof}{Proof}[section]
\usepackage{graphicx}
\usepackage{subcaption}
\usepackage[ruled,vlined]{algorithm2e}
\usepackage{arxiv}

\usepackage[utf8]{inputenc} 
\usepackage[T1]{fontenc}    
\usepackage{hyperref}       
\usepackage{url}            
\usepackage{booktabs}       
\usepackage{amsfonts}       
\usepackage{nicefrac}       
\usepackage{microtype}      
\usepackage{lipsum}		
\usepackage{graphicx}
\usepackage{natbib}
\usepackage{doi}

\title{Hidden Representation Clustering with Multi-Task Representation Learning towards Robust Online Budget Allocation}

\author{
    Xiaohan Wang, Yu Zhang, Guibin Jiang, Bing Cheng, Wei Lin\\
    \normalsize{Meituan} \\
    \normalsize{Beijing, China} \\
    \normalsize{\{wangxiaohan17, zhangyu409, jiangguibin, bing.cheng, linwei31\}@meituan.com}
}

\renewcommand{\shorttitle}{\textit{arXiv} Template}

\begin{document}
\maketitle

\begin{abstract}
Marketing optimization, commonly formulated as an online budget allocation problem, has emerged as a pivotal factor in driving user growth. Most existing research addresses this problem by following the principle of 'first predict then optimize' for each individual, which presents challenges related to large-scale counterfactual prediction and solving complexity trade-offs. Note that the practical data quality is uncontrollable, and the solving scale tends to be tens of millions. Therefore, the existing approaches make the robust budget allocation non-trivial, especially in industrial scenarios with considerable data noise. To this end, this paper proposes a novel approach that solves the problem from the cluster perspective. Specifically, we propose a multi-task representation network to learn the inherent attributes of individuals and project the original features into high-dimension hidden representations through the first two layers of the trained network. Then, we divide these hidden representations into $K$ groups through partitioning-based clustering, thus reformulating the problem as an integer stochastic programming problem under different total budgets. 
Finally, we distill the representation module and clustering model into a multi-category model to facilitate online deployment. Offline experiments validate the effectiveness and superiority of our approach compared to six state-of-the-art marketing optimization algorithms. Online A/B tests on the Meituan platform indicate that the approach outperforms the online algorithm by 0.53\% and 0.65\%, considering order volume (OV) and gross merchandise volume (GMV), respectively.
\end{abstract}


\section{Introduction}

As one of the most cost-effective business activities in E-commerce and life service platforms, marketing optimization plays an increasingly significant role in promoting user growth \cite{fang2024backdoor, albert2022commerce}. The essence of marketing optimization is to facilitate the increase in revenue, such as order volume (OV) and gross merchandise volume (GMV) within limited cost, which tends to be formulated as a budget allocation problem \cite{vanderschueren2024metalearners}. However, revenue and cost values usually change dynamically caused by market fluctuations, and the resulting changes in the optimization parameters make it non-trivial to directly resolve the problem using mathematical optimization. 

Most existing research solves this challenge following the principle of 'first prediction then optimization', and the resulting approaches are usually named two-stage methods \cite{de2024uplift}. The first stage usually predicts the revenue and cost through machine learning methods, and then the second stage optimizes the budget allocation problem with operational research (OR) based on predicted parameters. As prediction tends to be counterfactual, numerous studies adopt causal inference methods such as causal forest (CF) and Dragonnet. However, the theoretical optimal solution in OR struggles to remain global optimum because of the accumulative errors. The practical performance of two-stage methods may even be inferior to heuristic searching methods \cite{zhou2023direct}. To alleviate this issue, decision-focused learning (DFL) proposes to introduce the decision loss from the OR optimization to training, thus integrating the two-stage methods into an end-to-end system \cite{wilder2019melding}. Related research usually concentrates on stably estimating the gradient of the decision loss since the OR objective function is non-differentiable. For instance, recent work applies Lagrangian duality and policy learning to DFL and achieves superior performance in practical scenarios \cite{zhou2024decision}. 

Nonetheless, the existing methods for marketing optimization usually face two key challenges:

1) \textbf{Noise influence}. The prediction task is dedicated to forecasting individuals' counterfactual revenues and costs based on historical data. However, the individual-level prediction is extremely sensitive to the noise in data, which puts a higher requirement on the quality of collected data. Collecting pure data without noise is unrealistic, especially in practical scenarios with high GMV variance, such as hotel pricing. For example, an individual may exhibit distinguishable marketing attributes when paying for the hotel at public expense or self-funded, which confuses the model during training and leads to unsatisfactory algorithmic performance.

2) \textbf{Large-scale counterfactual prediction}. In practical marketing optimization, the individuals' quantity for prediction tends to be tens of millions. Additionally, we can not collect historical data for each individual with all treatments, making the prediction a counterfactual task. The massive individual quantity in practical scenarios and the counterfactual prediction make accurate predictions of individual-level revenues and costs non-trivial, thus resulting in the decrease in marketing optimization.

3) \textbf{Time-presicion trade-off}. The time complexity of OR increases exponentially with the number of individuals and treatments, resulting in the solving time under large-scale individuals becomes unacceptable even in offline computation. Besides, current research also finds that the global optimal solution may overfit the marketing optimization considering the prediction error in the first stage \cite{mandi2022decision}. To this end, researchers usually apply Lagrangian duality or heuristic searching to replace the origin OR, thus accelerating the offline solving speed. However, these methods need to balance the time-consumption and the solving precision. A second stage with fast solving-speed usually faces low solving-presicion issues.

To cope with the above challenges in existing market optimization, this paper proposes to solve the online resource allocation problem from the perspective of clusters rather than individuals. The core idea is to learn individuals' inherent preferences rather than specific prediction values. Specifically, we first propose the multi-task monotonic network to model individuals' preferences with historical data. Then, the hidden representations of the network are separated for clustering individuals. Based on the statistical values of clusters, we leverage stochastic programming to solve the allocation strategies under budgets in seconds. Finally, we distill the hidden representation module and the clustering model into a multi-classification model for online deployment, making the entire online link respond to immediate requests in milliseconds.

The contributions of the paper are summarized as follows:

1) We propose a cluster-based marketing optimization approach to simultaneously cope with the challenges of noise influence, large-scale counterfactual prediction, and time-precision trade-off. Compared to existing research that adopts individual-level prediction, we first propose group-level clustering and successfully apply it in practical scenarios.

2) We propose hidden representation clustering (HRC) to improve the robustness of marketing optimization and model the uplift between multi-valued treatments based on clusters rather than individuals. Compared to individual-level methods such as two-stage methods and DFL, HRC avoids the direct prediction of large-scale counterfactual revenues and thus remains robust to data noise.

3) To effectively train the hidden representation module, we propose a multi-task monotonic network to learn the inherent attributes of individuals. The network can be trained alternatively based on randomized controlled trial (RCT) or observation (OBS) data. For OBS training data, we adopt the hypernetwork to add monotonic constraints of multi-valued treatments.

4) We introduce stochastic programming with variance aversion to solve the optimal strategies under different budgets in the second stage. We leverage the statistical values of the clusters rather than the predicted values of the individuals as the parameters for solving, which greatly reduces the complexity of solving while maintaining accuracy with respect to the dataset.

5) Experiment results on practical scenarios validated the effectiveness and superiority of HRC compared to the other six mainstream marketing optimization methods. Online A/B test indicates that the proposed HRC can stably outperform two-stage methods and DFL.

The rest of the paper is organized as follows: The related work on the existing research of two-stage methods and DFL is reviewed in Section 2. The problem formulation of the studied online resource allocation problem is described in Section 3. Section 4 introduces algorithmic details of the proposed approach in steps. Section 5 presents the offline and online experimental results, and Section 6 concludes the paper.

\section{Related Work}
\subsection{Two-stage Methods}
Two-stage methods follow the principle of 'first predict then optimize', which has been widely studied in industrial applications, including marketing optimization \cite{geng2024benchmarking, sun2024end, he2024rankability, vanderschueren2024metalearners}, E-commerce \cite{albert2022commerce}, and advertising bidding \cite{betlei2024maximizing}. The prediction stage usually models the input through machine learning and causal inference, such as causal forest \cite{ai2022lbcf}, S-learner \cite{chen2024practical}, and causal representation learning \cite{shi2019adapting, nie2021vcnet}. Then, the optimization stage solves the allocation problem under budget constraints through operation research (OR) based on parameters predicted by the first stage. Since decision variables often number in the millions, most of the works accelerate the resolving speed through Lagrangian duality \cite{zhou2023direct}. Additionally, previous studies have validated the effectiveness of heuristic search with expert knowledge in improving the generalization \cite{kamran2024learning, ai2024improve}. In recent years, the related work on two-stage methods has concentrated on monotonic constraints \cite{wang2023multi, sun2024end}, continuous treatment optimization \cite{de2024uplift}, long-tail label distributions \cite{he2024rankability}, ranking loss functions \cite{kamran2024learning}, and dynamic operational contexts \cite{vanderschueren2024metalearners}. Nonetheless, little research considers the prediction and optimization under noisy data, which may limit their application in practical industrial scenarios with uncertain data noise.

\subsection{Decision-Focused Learning}
Compared to two-stage methods, decision-focused learning (DFL) introduces the decision loss from the optimization stage to the training of the prediction model \cite{wilder2019melding}. As most of objective functions in two-stage methods are non-differentiable, the primary challenge of DFL concentrates on unbiased gradient estimation. For example, recent research alleviates non-differentiable challenge from perspectives of policy learning \cite{zhou2024decision}, interior point \cite{mandi2010interior}, implicit maximum likelihood estimation \cite{niepert2021implicit}, and differentiable optimizers \cite{berthet2020learning, vlastelica2019differentiation}. Although these studies have demonstrated the unbiasedness of the gradient estimation of the decision loss function, the high variance caused by the estimation may lead to training fluctuations. To tackle this issue, the existing research propose to improve the robustness of DFL through rank-score learning \cite{mandi2022decision, zhou2023direct}, black-box optimization \cite{vlastelica2019differentiation}, surrogate loss \cite{shah2022decision}, and stochastic programming \cite{donti2017task}. Nonetheless, most existing work on DFL overlooks the consideration of data noise in training and optimization data. Although the introduced decision loss helps the prediction model learn end-to-end knowledge from the optimization stage, it also increases the risk of overfitting to the optimization data simultaneously.

\section{Problem Formulation}

This section presents the mathematical formulation of the online budget allocation problem. Suppose $r_{i,j}$ and $c_{i,j}$ represent revenue and cost of the $i$-th individual under the $j$-th treatment, where $i \in \{1,2,..., N\}$ and $j \in \{1,2,..., M\}$. Given a total budget of $B$, the online budget allocation is formulated as an integer programming problem as follows:
\begin{equation} \label{eqp1}
\begin{aligned}
    \max &\sum_{i=1}^{N} \sum_{j=1}^{M} r_{i,j} \cdot x_{i,j} \\
    s.t. &\sum_{i=1}^{N} \sum_{j=1}^{M} c_{i,j} \cdot x_{i,j} \leq B \\
    &\sum_{j=1}^{M} x_{i,j} = 1, \forall i \in \{1,2,...,N\} \\
    &x_{i,j} \in \{0, 1\}, \forall i \in \{1,2,..., N\}, \forall j \in \{1,2,..., M\}
\end{aligned}
\end{equation}
where $x_{i,j}$ denotes the binary decision variable indicating whether to assign the $j$-th treatment to the $i$-th individual. The budget $B$ usually fluctuates greatly in a practical marketing environment, so this formulation will be resolved multiple times to suit distinguishable budgets. 

\section{Proposed Method}

\subsection{Overall Framework}
\begin{figure}[h]
\centering
  \includegraphics[width=\linewidth]{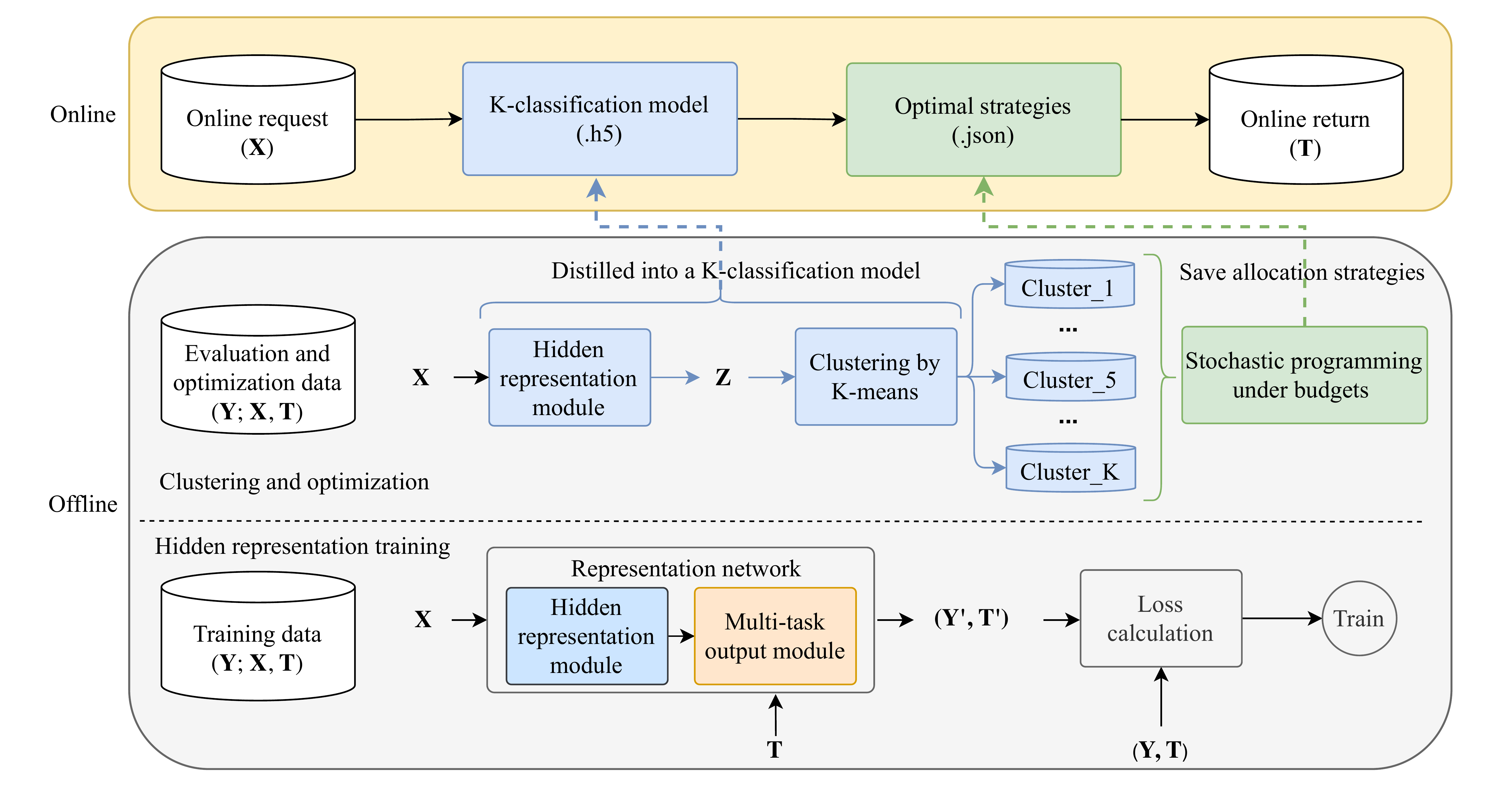}
  \caption{The overall framework of the proposed HRC.}
  \label{fig:f1}
\end{figure}

Figure \ref{fig:f1} indicates the overall framework that contains offline computation and online deployment. The offline computation includes three steps: First, we train a multi-task representation network with collected RCT data or OBS data. Then, the hidden representation module of the trained network projects the original individual features into hidden representations, and these individuals are clustered into $K$ groups according to these hidden representations. Finally, stochastic programming is employed to resolve optimal allocation strategies for $K$ groups under diverse budgets. These three steps will be introduced in detail by Sections \ref{sec1}, \ref{sec2}, and \ref{sec3}, respectively.

For online deployment, we distill the hidden representation module and the clustering model into one K-classification model. Therefore, an online request with the feature $x$ can be quickly predicted to a group index within the scope of $\{1,2,..., K\}$. Then, the saved allocation strategies map the index to the optimal treatment assignment. Such a simple online workflow ensures the immediate execution of millions of requests in milliseconds.

\subsection{Multi-Task Representation Network} \label{sec1}
\begin{figure}[h]
\centering
  \includegraphics[width=0.6\linewidth]{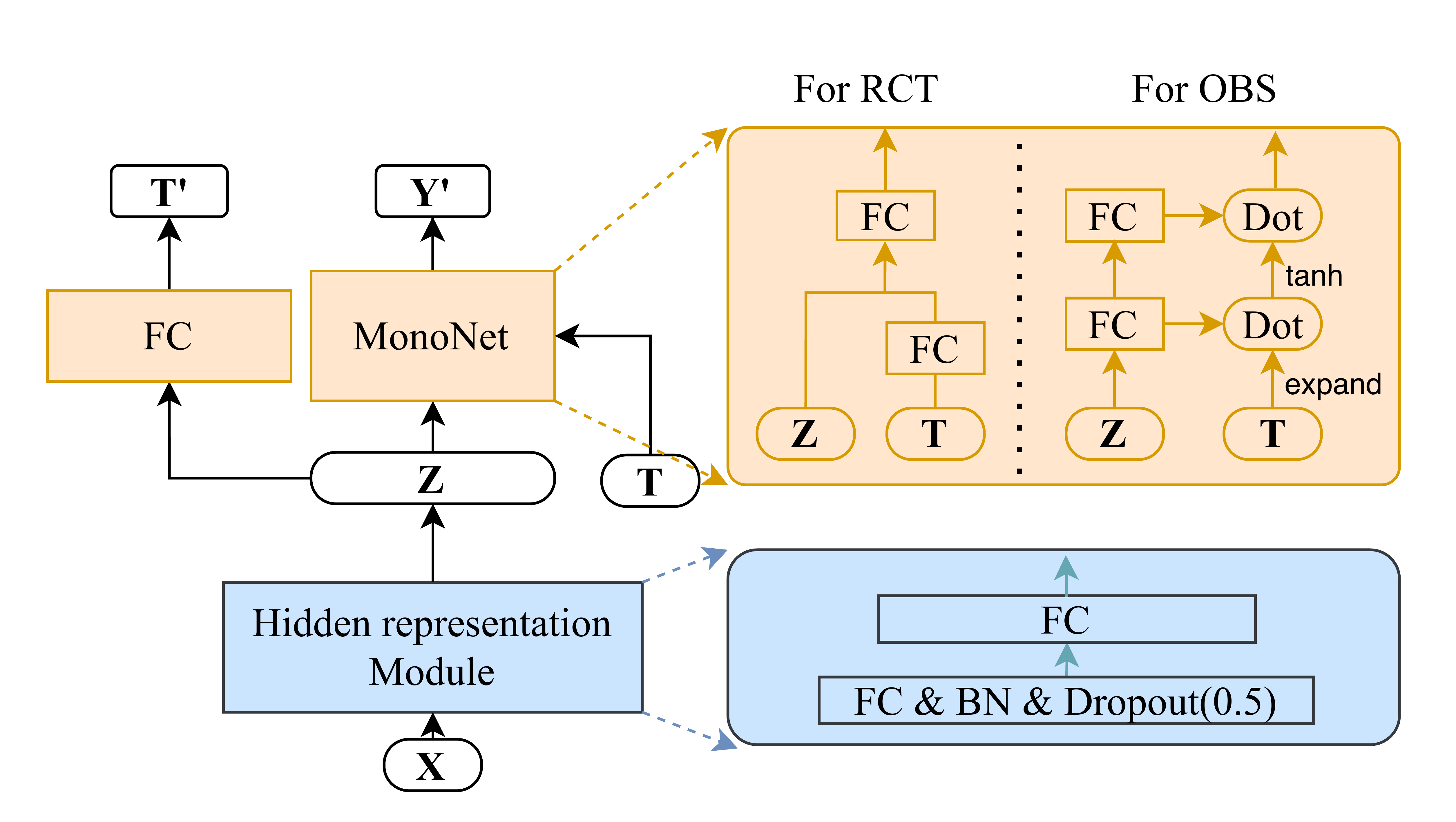}
  \caption{The model structure of the multi-task representation network.}
  \label{fig:f2}
\end{figure}

Figure \ref{fig:f2} indicates the structure of the multi-task representation network with two primary modules. The hidden representation module with parameters $\theta$ projects the original features $X$ into hidden representations $Z^{nn}(X; \theta)$. Then the multi-task output module predicts revenues $R^{nn}(T, Z; \phi_1)$ and propensity scores $g^{nn}(Z; \phi_2)$ based on $Z^{nn}(X; \theta)$, where $\phi_1, \phi_2$ represent parameters of revenue and propensity heads, respectively. The main structure of the representation network is similar to Dragonnet \cite{shi2019adapting}, and VCnet \cite{nie2021vcnet}. The representation network is trained by minimizing (\ref{eqm1}).
\begin{equation} \label{eqm1}
    \hat{\theta} = \mathop{argmin}\limits_{\theta, \phi_1, \phi_2} \frac{1}{N} \sum_{i=1}^{N} [(R^{nn}(t_i, z_i; \phi_1) - y_i)^2 + \alpha (g^{nn}(z_i; \phi_2) - t_i)^2 ]
\end{equation}
where $\alpha > 0$ represents a hyperparameter weighting the propensity loss. The training target is to optimize the parameters $\theta$ to capture the hidden representations of $X$ effectively. Regarding the revenue head $R^{nn}(T, Z; \phi_1)$, we provide two alternative networks, one for RCT data and another for observation data, separately. As indicated in the upper right corner of \ref{fig:f2}, we directly concatenate the embedding of $T$ and the hidden representation $Z$ for RCT training data. For observation training data, we employ the hypernetwork represented by equation (\ref{eqm2}).
\begin{equation} \label{eqm2}
    R^{nn}(T, Z; \phi_1) = |\omega^{nn}(Z; \phi_{1\_2})| \cdot \tanh(|\omega^{nn}(Z; \phi_{1\_1})| \cdot Repeat(T)) 
\end{equation}
where $|\omega^{nn}(\cdot; \phi_{1\_x})|, x \in \{1,2\}$ denotes a neural network with parameters $\phi_{1\_x}$ projecting $Z$ to a positive weight matrix. $Repeat(T)$ expands the dimension of $T$ by repeating the value of $T$, i.e., $Repeat(t = 1) \rightarrow [1, 1 , ..., 1]$. With the assistance of (\ref{eqm2}), monotonicity can be enforced through a constraint on the relationship between $R^{nn}(t, Z; \phi_1)$ and $t$ as indicated by (\ref{eqm3})
\begin{equation} \label{eqm3}
    \frac{\partial R^{nn}(t, Z; \phi_1)}{\partial t}  > 0
\end{equation}
The predicted revenue is proportional to the treatment value with this monotonic constraint, which will be helpful for multi-treatment counterfactual prediction problems based on biased OBS data.

\subsection{Hidden Representation Clustering} \label{sec2}

Instead of grouping $X$ directly based on output $R^{nn}(T, Z; \phi_1)$ like two-stage methods, we utilize the trained hidden representation module to project $X$ to $Z^{nn}(X; \theta)$. The detailed theoretical analysis of hidden representation clustering compared to direct output clustering can be found in Appendix. A. 

The hidden representations $Z$ are clustered into $K$ groups as indicated in the middle level of Figure \ref{fig:f1}. The distribution distance $d(z_1, z_2)$ is defined by $d(z_1, z_2) = || z_1 - z_2||_2$, and K-Means is implemented as a quantization function $Q_z: \mathcal{Z} \rightarrow \{1,2,...,K\}$ as introduced by (\ref{eqm3_1}).
\begin{equation} \label{eqm3_1}
    Q_z(z) = \arg \min_k ||z - \mu_k||^2_F
\end{equation}
where $\{\mu_k\}_{k=1}^K$ represents the clustering centers of $K$ clusters. For each cluster, the mean and variance of revenues and costs are calculated under $M$ treatments. These statistical values are denoted by (\ref{eqm4}).
\begin{equation} \label{eqm4}
\{\hat{\mu}_{r}^{(i,j)}, \hat{\sigma}_{r}^{(i,j)} , \hat{\mu}_{c}^{(i,j)}, \hat{\sigma}_{c}^{(i,j)}, \omega_i\}, \forall i \in \{1,2,...,K\}, \forall j \in \{1,2,..., M\}
\end{equation}
where $\hat{\mu}_{r}^{(i,j)}, \hat{\sigma}_{r}^{(i,j)} , \hat{\mu}_{c}^{(i,j)}, \hat{\sigma}_{c}^{(i,j)}, \omega_i$ specifies the revenue mean, the revenue variance, the cost mean, the cost variance, and the individual quantity in the $i$-th cluster under the $j$-th treatment, respectively. 

\subsection{Stochastic Programming with K Clusters} \label{sec3}

Based on statistical values of $K$ clusters, we reformulate the original integer programming problem in (\ref{eqp1}) as a stochastic programming problem in (\ref{eqp2}).
\begin{equation} \label{eqp2}
\begin{aligned} 
    \max &\sum_{i=1}^{K} \sum_{j=1}^{M} \omega_i \cdot x_{i,j} \cdot (\hat{\mu}_{r}^{(i,j)}  - \lambda \hat{\sigma}_{r}^{(i,j)} - \kappa \hat{\sigma}_{c}^{(i,j)}) \\
    s.t. &\sum_{i=1}^{K} \sum_{j=1}^{M} \omega_i \cdot \hat{\mu}_{c}^{(i,j)} \cdot x_{i,j} \leq B \\
    &\sum_{j=1}^{M} x_{i,j} = 1, \forall i \in \{1,2,...,K\} \\
    &x_{i,j} \in \{0, 1\}, \forall i \in \{1,2,..., K\}, \forall j \in \{1,2,..., M\}
\end{aligned}
\end{equation}
where $\lambda, \kappa$ represent two factors determining the risk aversion of variance. Generally, the magnitude of $K$ is much smaller than that of $N$. Therefore, the solving speed of the stochastic programming problem of (\ref{eqp2}) is much faster than the original integer programming problem of (\ref{eqp1}). As a result, optimal allocation strategies under diverse budgets $B$ could be quickly resolved and saved as a mapping library for online requests.

\subsection{Multi-Classification Distillation}

Finally, we distill the hidden representation module and the K-Means model into one K-category neural network model by minimizing the cross entropy loss of (\ref{eqm5}).
\begin{equation} \label{eqm5} 
    \hat{\Phi} =\mathop{argmin}\limits_{\Phi} -\frac{1}{N} \sum_{i=1}^{N} \sum_{j=1}^{K}y_{i,j} \log(p^{nn}(c=j|x_i; \Phi))
\end{equation}
where $p^{nn}(c=j|x_i; \Phi)$ represents the probability of predicting $x_i$ as the $j$-th group, $\Phi$ denotes the parameters of the network. The distillation model integrates these two models into one, thus 1) simplifying the online deployment and 2) speeding up online responses to requests. Additionally, the distilled model can be easily fine-tuned with incremental data.

\section{Experiments}

This section validates the effectiveness and superiority of the proposed HRC through 1) offline comparisons with mainstream marketing optimization algorithms and 2) A/B tests with online-deployed algorithms in the Meituan platform.

\subsection{Dataset and Comparison Methods}

\textbf{Dataset} Five weeks of industrial data are collected from the Meituan platform for offline comparison, and a brief introduction of these data is reported in \autoref{expt1}. All original values from the Meituan platform are scaled and masked regarding confidential issues. As indicated by \autoref{expt1}, the data distributions between the Winter and Summer periods are distinguishable. The data are collected from practical shop pricing scenarios in the Meituan platform, where online shops provide daily discounts to consumers. The discount of an item in a shop for all consumers is the same to avoid price discrimination, but the discount rates vary from item to item and usually fluctuate by day. Six discount rates of $\{-5\%,-6\%,-7\%,-8\%,-9\%,-10\%\}$ serve as $6$ treatments. The datasets contain $44$ millions of samples in total, and each sample has $450$ features, one treatment label, one cost label, and $2$ revenue labels (OV and GMV).

\begin{table}
	\caption{Five weeks of data collected from the Meituan platform. Values of cost, OV, and GMV represent mean / std.}
	\centering
	\begin{tabular}{ccccccc}
		\toprule
		Name & Description   & Used for & Cost & OV & GMV  & Size \\
		\midrule
		Meituan-w1 & RCT, Winter& Train & 0.321/3.639 & 0.851/2.231 &  0.811/3.949 & (6e6, 4e2) \\
		Meituan-w2 & RCT, Winter& Test & 0.465/4.438 & 0.843/2.239 & 0.823/4.108  & (7e6, 4e2) \\
            Meituan-w3 & RCT, Summer& Test & 1.397/7.613 & 0.928/2.484 &  1.211/5.206 & (1e7, 4e2) \\
		Meituan-w4 & OBS, Winter& Train & 0.699/4.519 & 0.919/2.33 &  0.846/3.546 & (1.7e7, 4e2) \\
		Meituan-w5 & RCT, Winter& Test & 0.718/4.557 & 0.921/2.36 & 0.85/3.554  & (4e6, 4e2) \\
		\bottomrule
	\end{tabular}
	\label{expt1}
\end{table}

\textbf{Comparison methods} We compare six other SOTA marketing optimization algorithms in offline experiments, including both two-stage methods and DFLs. These methods include decision-focused learning with policy learning (DFL-PL), decision-focused learning with maximum entropy regularizer (DFL-MER), single learner with heuristic search (HEU-Slearner), single learner with Lagrangian duality (LANG-Slearner), causal forest with heuristic search (HEU-CF), and causal forest with Lagrangian duality (LANG-CF). To verify the causal debiasing ability of the proposed monotonic network, we add two extra algorithms, namely Dragonnet with Lagrangian duality (LANG-Drag) and Dragonnet with heuristic search (HEU-Drag) \cite{shi2019adapting}. The implementation details and training hyper-parameters of these methods can be found in Appendix B.


\textbf{Metrics} The offline comparison mainly utilizes the expected outcome metric (EOM) to simulate the marketing performance of offline experiments. EOM is a popular metric for marketing optimization in industry \cite{ai2022lbcf, zhao2017uplift, zhou2023direct, zhou2024decision}. EOM can derive an unbiased estimate of revenues (such as OV or GMV) for any given budget allocation strategy, thus simulating the offline performance of marketing optimization algorithms with collected RCT data. We use EOM as the primary metric to compare performances in offline datasets.

\subsection{Offline Evaluation: Comparison on Meituan Datasets}


We compare the algorithms' performance in Meituan datasets from three perspectives: solution quality, generalization to diverse data distributions, and causal inference ability under observational data. 

\textbf{Solution Quality.} In order to evaluate the solution quality of algorithms, we train models based on Meituan-w1 and evaluate their EOMs based on Meituan-w2. Note that the cost and revenue distributions of Meituan-w1 and Meituan-w2 are similar as listed by \autoref{expt1}. We train models regarding OV as the revenue objective and evaluate both the OV and GMV performances based on EOM. The corresponding EOM results of OV and GMV are indicated by two subfigures in \autoref{expf1}, respectively. Under six budget points within the scope of $(0.4, 0.6)$, HRC achieves the most outstanding performance regarding OV-EOM lines in \autoref{expf1ov}. Additionally, the GMV of HRC also indicates competitive results compared to the other six algorithms in \autoref{expf1gmv}. 

\begin{figure}[htbp]
    \centering
    \subcaptionbox{OV \label{expf1ov}}{\includegraphics[width=0.4\textwidth]{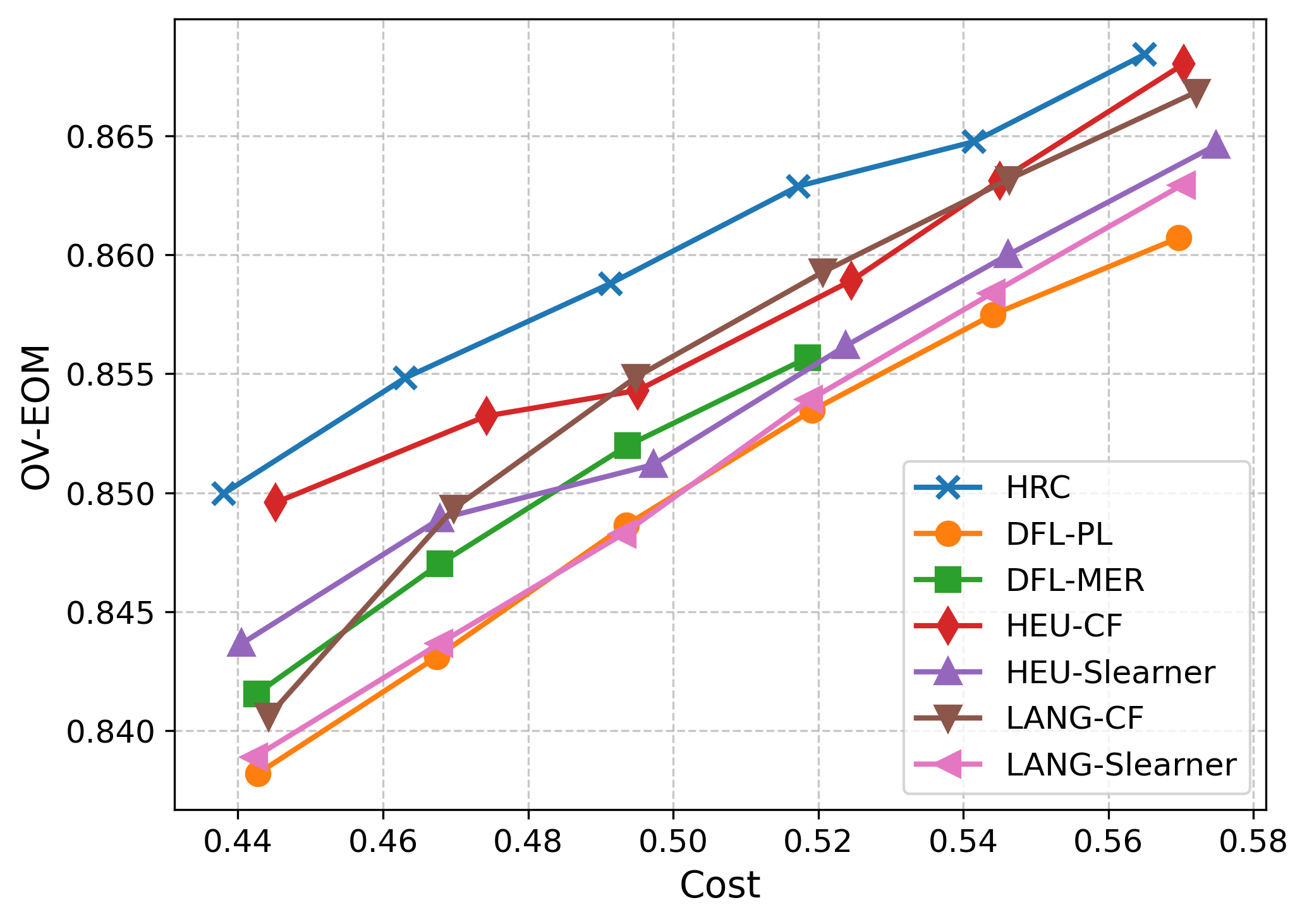}} 
    \subcaptionbox{GMV label\label{expf1gmv}}{\includegraphics[width=0.4\textwidth]{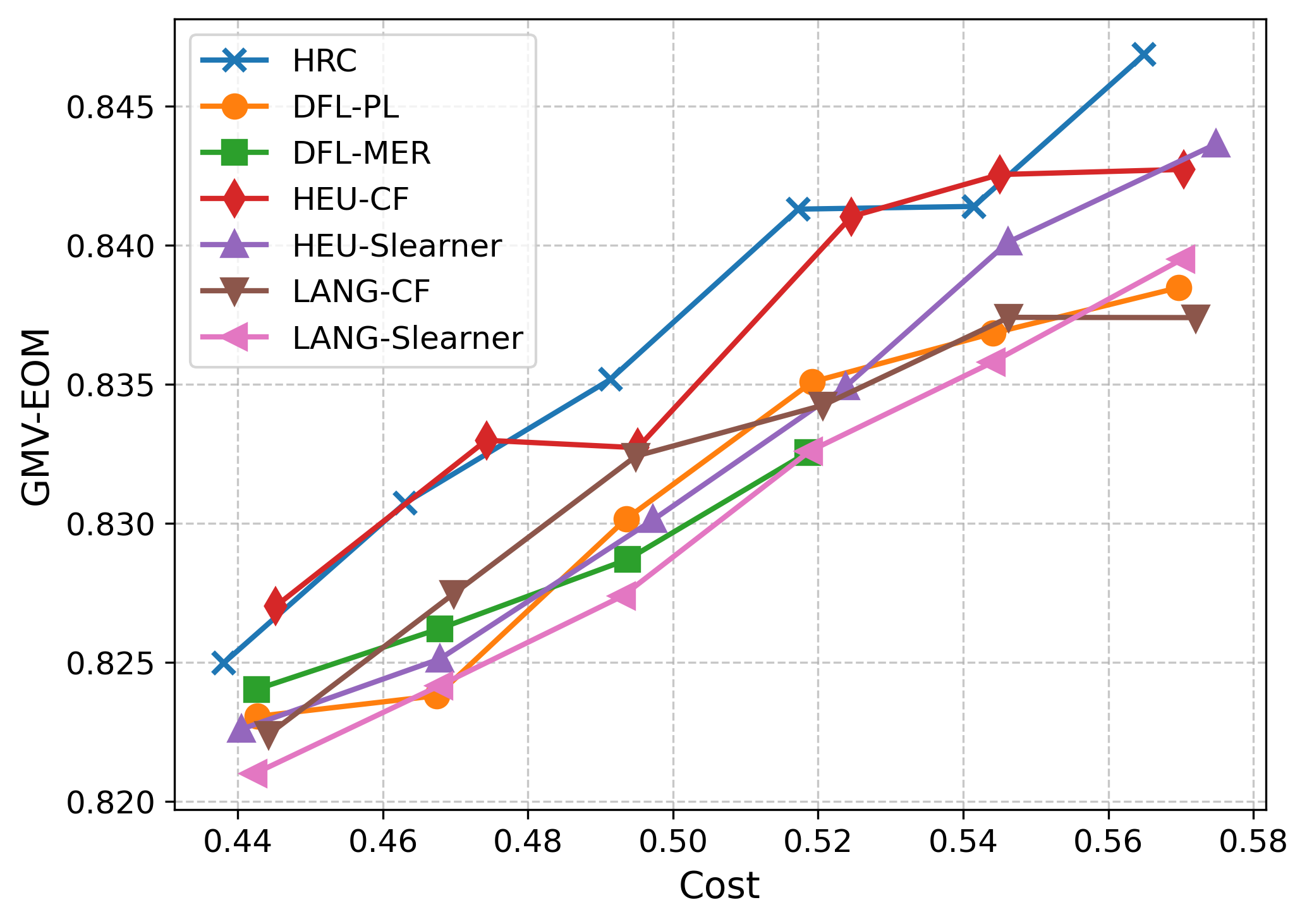}} 
    \caption{EOM curves of OV and GMV under different budgets}
    \label{expf1}
\end{figure}

\textbf{Generalization.} To verify the generalization to distinguishable data distributions, we evaluate the EOM performances of OV and GMV in Meituan-w3. Note that both cost and revenue distributions between Meituian-w1 (training dataset) and Meituan-w3 (evaluation dataset) are in difference. \autoref{expt2} indicates the EOM results under five budget points, where HRC outperforms the second-best approach for $\{0.1,0.1,0.1,0.1,0.1\}$ under budgets of $\{1,2,3,4,5\}$, respectively. To compare the general performance, we plot the box lines of OV and GMV EOMs in \autoref{expf2}, where HRC shows the highest average OV/GMV. Therefore, the generalization of HRC is still competitive compared to other methods, which ensures the stability of performance when the budget environment varies a lot.

\begin{figure}[htbp]
    \centering
    \subcaptionbox{OV \label{expf2ov}}{\includegraphics[width=0.4\textwidth]{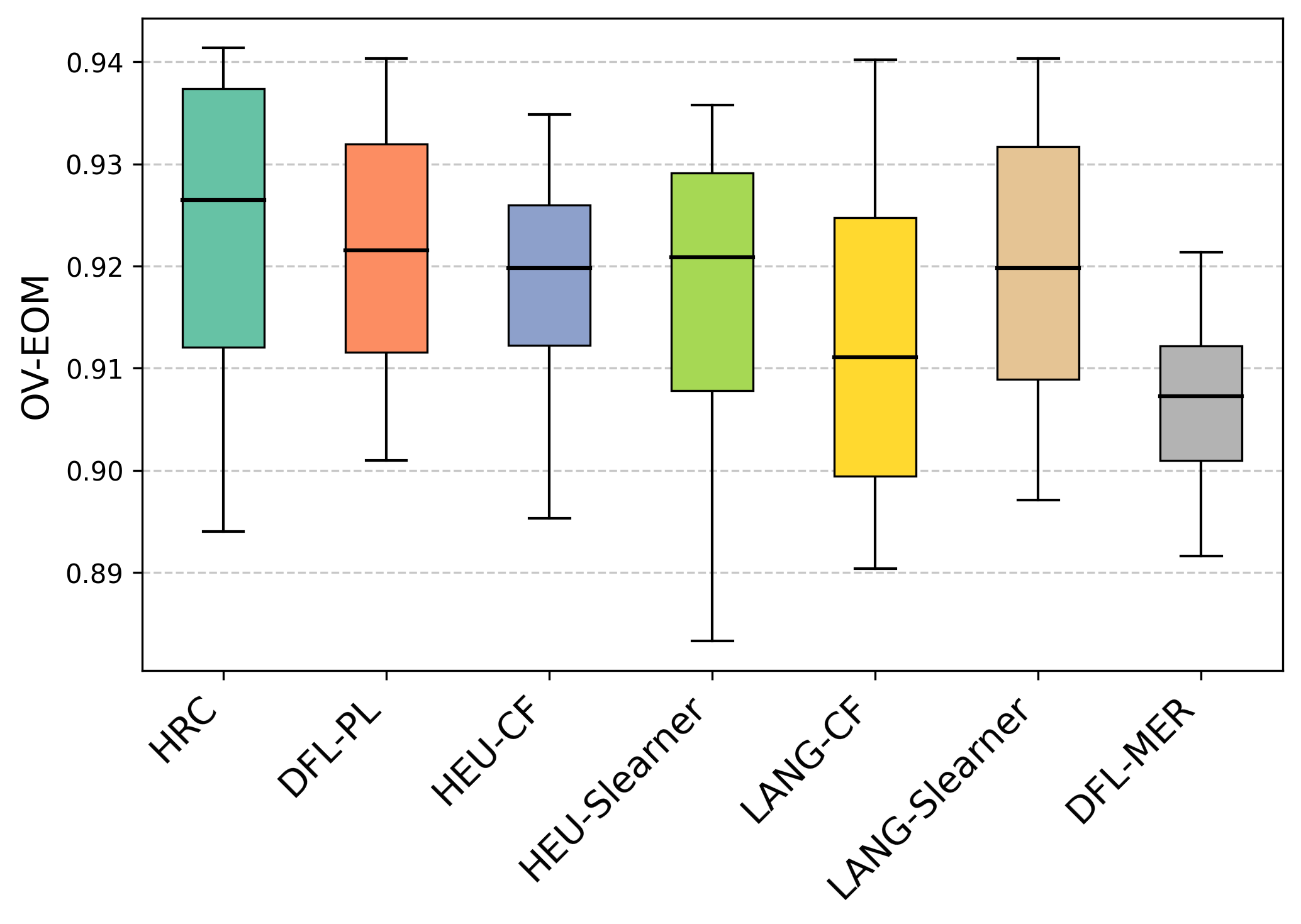}} 
    \subcaptionbox{GMV label\label{expf2gmv}}{\includegraphics[width=0.4\textwidth]{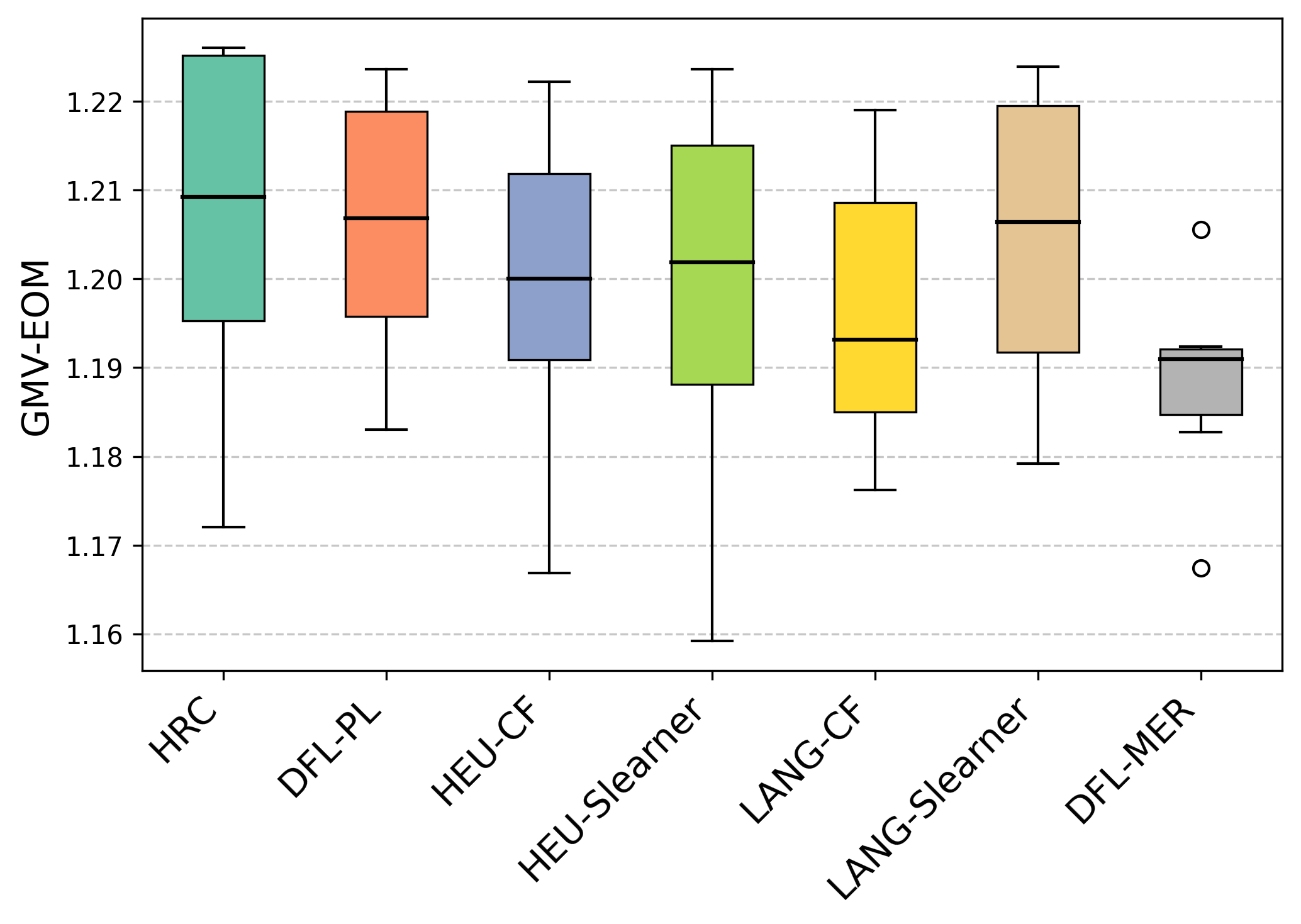}} 
    \caption{EOM boxlines of OV and GMV under different budgets}
    \label{expf2}
\end{figure}

\textbf{Performance under Observation Data.} Then, we evaluate the causal debiasing capability of the monotonic network when facing OBS data during training. In this comparison, we train models with Meituan-w4 collected from the online-algorithm results and evaluate the performance on Meituan-w5. \autoref{expf3} indicates the EOM results of OV and GMV, where HRC-Mono represents the proposed HRC with monotonic representation network. Both curves of OV and GMV indicate that the proposed HRC-mono outperforms the other four popular causal marketing optimization algorithms.

\begin{figure}[htbp]
    \centering
    \subcaptionbox{OV \label{expf3ov}}{\includegraphics[width=0.4\textwidth]{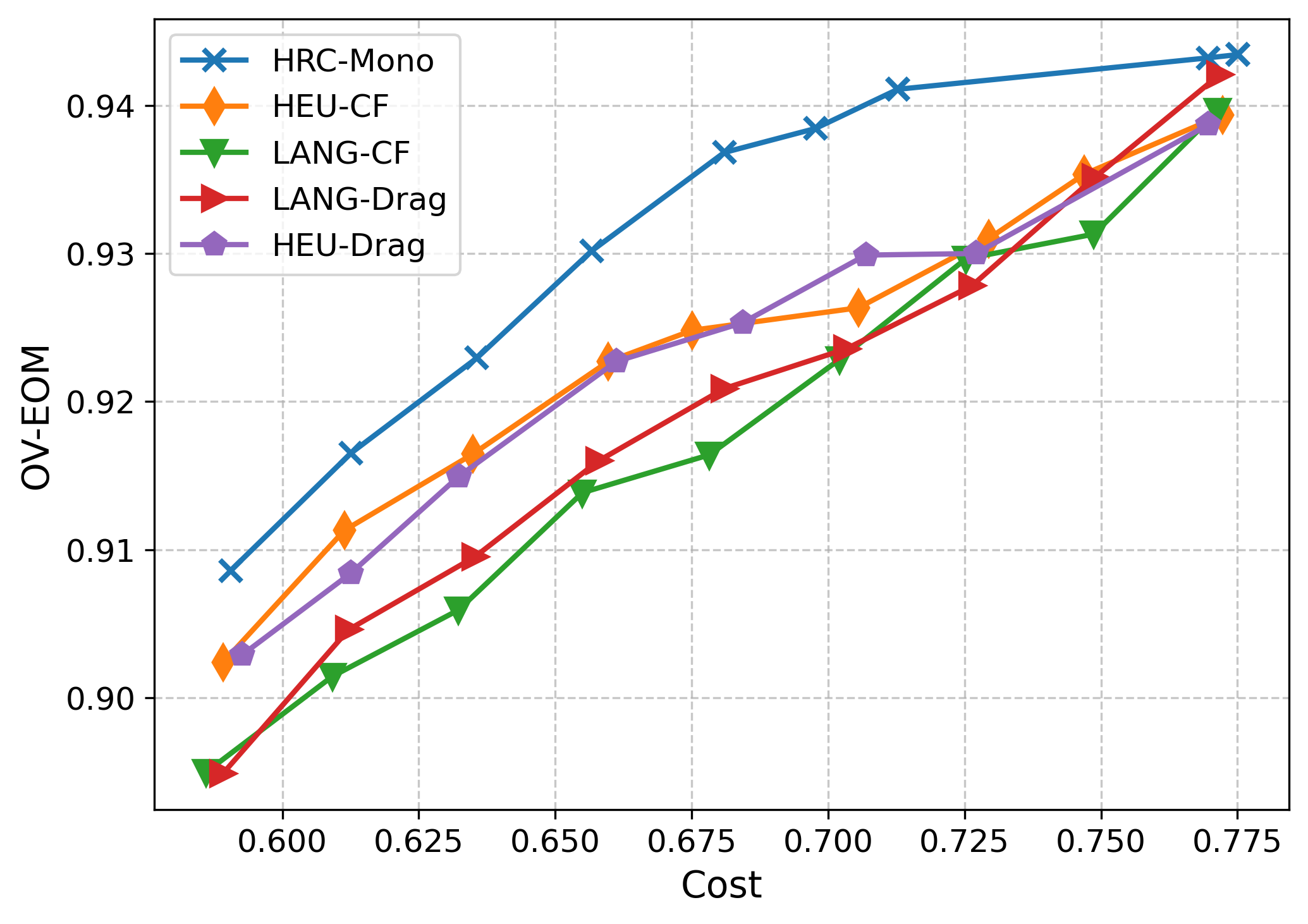}} 
    \subcaptionbox{GMV \label{expf3gmv}}{\includegraphics[width=0.4\textwidth]{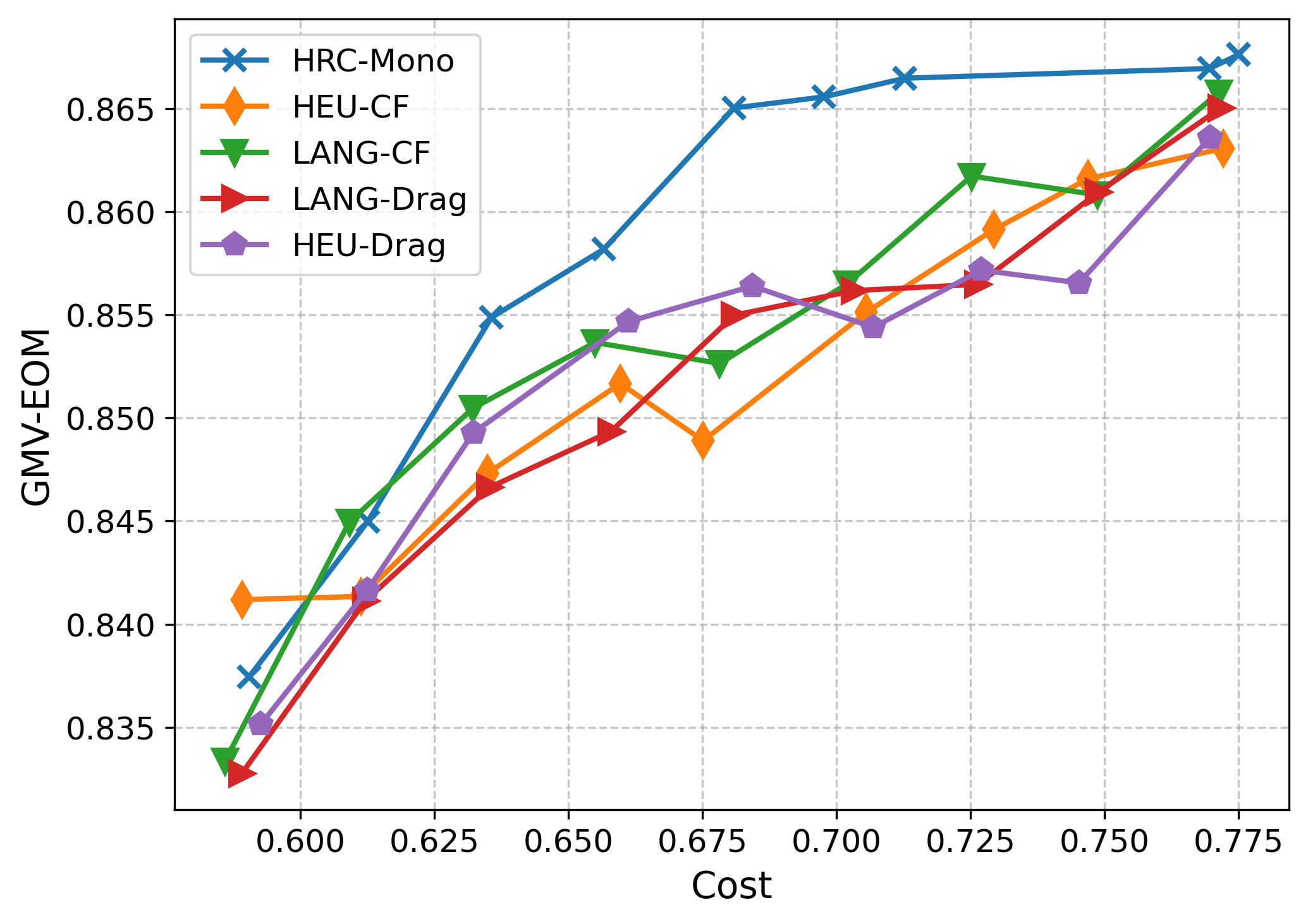}} 
    \caption{EOM curves of OV and GMV. The models are trained under observation data.}
    \label{expf3}
\end{figure}

\subsection{Online A/B Test}

\begin{table}
	\caption{Statistical results of online A/B tests}
	\centering
	\begin{tabular}{ccccccc}
		\toprule
		A v.s. B & Period & $\overline{\Delta OV / OV_B}$   & $\overline{\Delta GMV / GMV_B}$ & $\overline{\Delta Cost / Cost_B}$ \\
		\midrule
		HRC v.s. Slearner-HEU & Week1$\sim$Week2 & 0.13\% & 0.23\% & 0.00296pp \\
		HRC v.s. DFL-PL & Week3$\sim$Week4 & 0.53\% & 0.65\% & 0.03043pp \\
		\bottomrule
	\end{tabular}
	\label{expt3}
\end{table}

Finally, we evaluate the practical performance of the proposed HRC on the online platform. To efficiently compare the practical performance, we set up two large-scale online data-flows for online comparisons of HRC and the other algorithm. The comparison algorithms include DFL-PL and Slearner-HEU, which have been deployed online for weeks. We record online OV and GMV gaps (\%) between HRC and Slearner-HEU for the first two weeks. Then, we leverage another two weeks to record the comparison results between HRC and DFL-PL. \autoref{expf4qfs} and \autoref{expf4dfl} indicate the daily results of HRC v.s. Slearner-HEU and DFL-PL, respectively. Those results in \autoref{expf4} indicate that HRC can stably outperform the other two algorithms in the practical Meituan scenario. \autoref{expt3} reports the statistical results in the total one-month of online observation. In terms of OV and GMV, HRC outperforms Slearner-HEU by 11\% and 22\%, respectively.
With respect to OV and GMV, HRC surpasses DFL-PL by 22\% and 33\%, respectively. Note that all the costs are controlled in a reasonable range to clip the budgets between the two methods.

\begin{figure}[htbp]
    \centering
    \subcaptionbox{Week1$\sim$Week2 \label{expf4qfs}}{\includegraphics[width=0.72\textwidth]{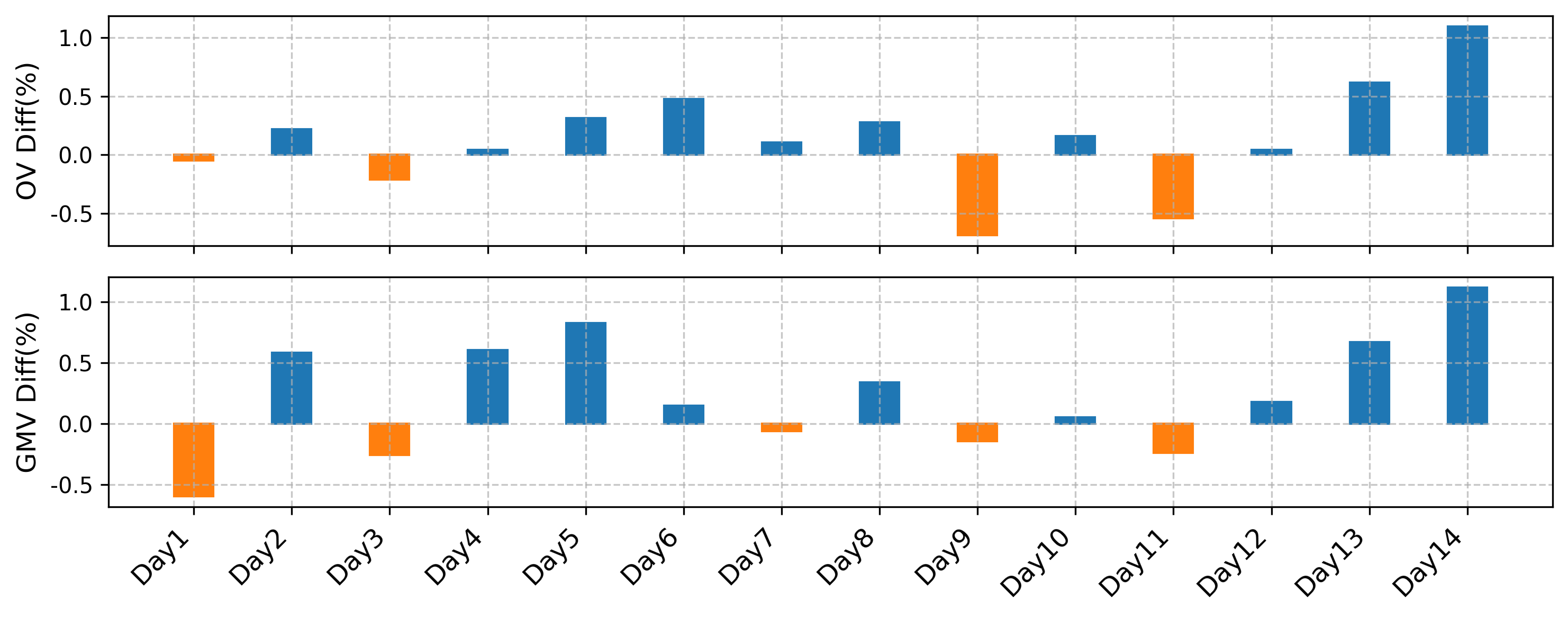}} 
    \subcaptionbox{Week3$\sim$Week4 \label{expf4dfl}}{\includegraphics[width=0.72\textwidth]{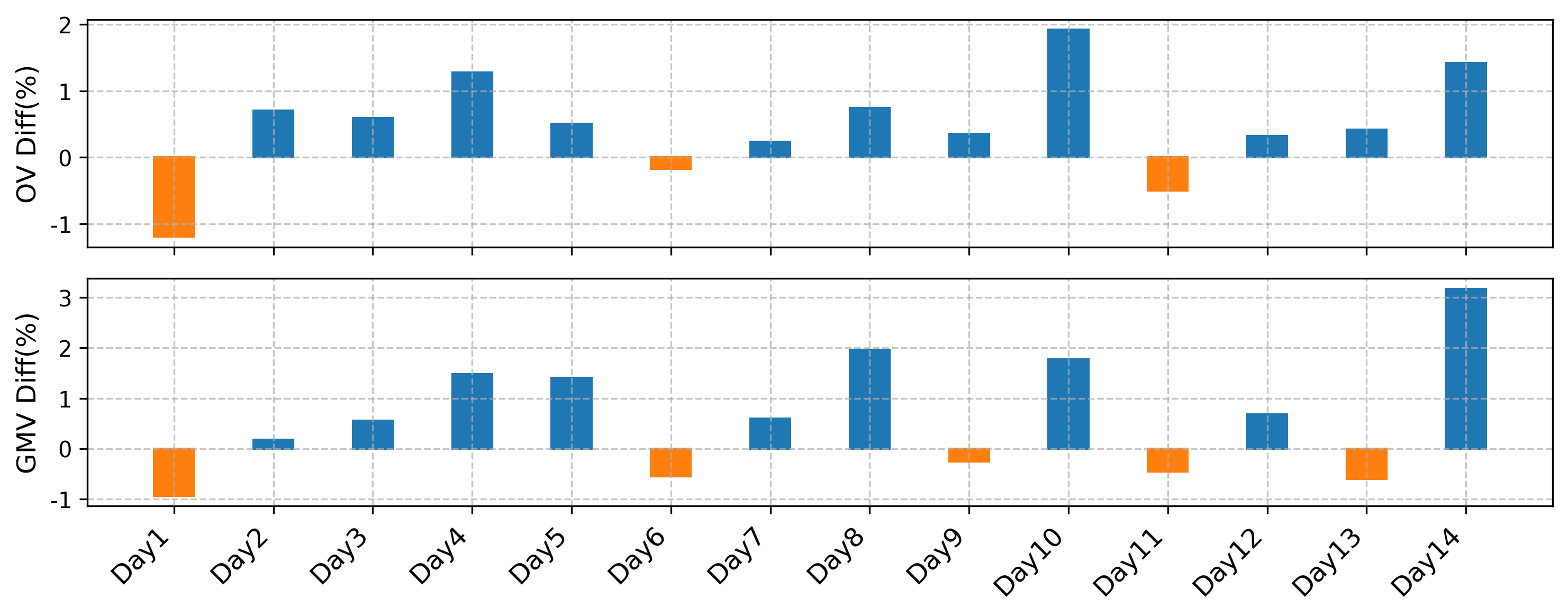}} 
    \caption{Daily online A/B results}
    \label{expf4}
\end{figure}

\section{Conclusion}

This paper proposes HRC for robust online budget allocation in marketing optimization. In contrast to existing research that models the problem from an individual perspective, this study leverages a cluster-based approach to cope with data noise in practical scenarios. A multi-task representation network is proposed to project original individuals' features into hidden representations for the following clustering. The stochastic programming based on $K$ clusters ensures the robustness of optimization while greatly improving the offline solving speed. Additionally, the distilled multi-classification models simplify the online deployment by facilitating the online response. Finally, offline and online comparison experiments validate the effectiveness and superiority of the proposed HRC. 

\bibliographystyle{unsrtnat}
\bibliography{references}

\begin{thebibliography}{28}
\providecommand{\natexlab}[1]{#1}
\providecommand{\url}[1]{\texttt{#1}}
\expandafter\ifx\csname urlstyle\endcsname\relax
  \providecommand{\doi}[1]{doi: #1}\else
  \providecommand{\doi}{doi: \begingroup \urlstyle{rm}\Url}\fi

\bibitem[Fang et~al.(2024)Fang, Zhang, Cui, Tang, Gu, Li, Gu, and Zhou]{fang2024backdoor}
Junpeng Fang, Gongduo Zhang, Qing Cui, Caizhi Tang, Lihong Gu, Longfei Li, Jinjie Gu, and Jun Zhou.
\newblock Backdoor adjustment via group adaptation for debiased coupon recommendations.
\newblock In \emph{Proceedings of the AAAI Conference on Artificial Intelligence}, volume~38, pages 11944--11952, 2024.

\bibitem[Albert and Goldenberg(2022)]{albert2022commerce}
Javier Albert and Dmitri Goldenberg.
\newblock E-commerce promotions personalization via online multiple-choice knapsack with uplift modeling.
\newblock In \emph{Proceedings of the 31st ACM International Conference on Information \& Knowledge Management}, pages 2863--2872, 2022.

\bibitem[Vanderschueren et~al.(2024)Vanderschueren, Verbeke, Moraes, and Proen{\c{c}}a]{vanderschueren2024metalearners}
Toon Vanderschueren, Wouter Verbeke, Felipe Moraes, and Hugo~Manuel Proen{\c{c}}a.
\newblock Metalearners for ranking treatment effects.
\newblock \emph{arXiv preprint arXiv:2405.02183}, 2024.

\bibitem[De~Vos et~al.(2024)De~Vos, Bockel-Rickermann, Lessmann, and Verbeke]{de2024uplift}
Simon De~Vos, Christopher Bockel-Rickermann, Stefan Lessmann, and Wouter Verbeke.
\newblock Uplift modeling with continuous treatments: A predict-then-optimize approach.
\newblock \emph{arXiv preprint arXiv:2412.09232}, 2024.

\bibitem[Zhou et~al.(2023)Zhou, Li, Jiang, Zheng, and Wang]{zhou2023direct}
Hao Zhou, Shaoming Li, Guibin Jiang, Jiaqi Zheng, and Dong Wang.
\newblock Direct heterogeneous causal learning for resource allocation problems in marketing.
\newblock In \emph{Proceedings of the AAAI Conference on Artificial Intelligence}, volume~37, pages 5446--5454, 2023.

\bibitem[Wilder et~al.(2019)Wilder, Dilkina, and Tambe]{wilder2019melding}
Bryan Wilder, Bistra Dilkina, and Milind Tambe.
\newblock Melding the data-decisions pipeline: Decision-focused learning for combinatorial optimization.
\newblock In \emph{Proceedings of the AAAI Conference on Artificial Intelligence}, volume~33, pages 1658--1665, 2019.

\bibitem[Zhou et~al.(2024)Zhou, Huang, Li, Jiang, Zheng, Cheng, and Lin]{zhou2024decision}
Hao Zhou, Rongxiao Huang, Shaoming Li, Guibin Jiang, Jiaqi Zheng, Bing Cheng, and Wei Lin.
\newblock Decision focused causal learning for direct counterfactual marketing optimization.
\newblock In \emph{Proceedings of the 30th ACM SIGKDD Conference on Knowledge Discovery and Data Mining}, pages 6368--6379, 2024.

\bibitem[Mandi et~al.(2022)Mandi, Bucarey, Tchomba, and Guns]{mandi2022decision}
Jayanta Mandi, V{\i}ctor Bucarey, Maxime Mulamba~Ke Tchomba, and Tias Guns.
\newblock Decision-focused learning: Through the lens of learning to rank.
\newblock In \emph{International conference on machine learning}, pages 14935--14947. PMLR, 2022.

\bibitem[Geng et~al.(2024)Geng, Ruan, Wang, Li, Wang, Chen, and Yan]{geng2024benchmarking}
Haoyu Geng, Hang Ruan, Runzhong Wang, Yang Li, Yang Wang, Lei Chen, and Junchi Yan.
\newblock Benchmarking pto and pno methods in the predictive combinatorial optimization regime.
\newblock \emph{Advances in Neural Information Processing Systems}, 37:\penalty0 65944--65971, 2024.

\bibitem[Sun et~al.(2024)Sun, Yang, Liu, Weng, Tang, and He]{sun2024end}
Zexu Sun, Hao Yang, Dugang Liu, Yunpeng Weng, Xing Tang, and Xiuqiang He.
\newblock End-to-end cost-effective incentive recommendation under budget constraint with uplift modeling.
\newblock In \emph{Proceedings of the 18th ACM Conference on Recommender Systems}, pages 560--569, 2024.

\bibitem[He et~al.(2024)He, Weng, Tang, Cui, Sun, Chen, He, and Ma]{he2024rankability}
Bowei He, Yunpeng Weng, Xing Tang, Ziqiang Cui, Zexu Sun, Liang Chen, Xiuqiang He, and Chen Ma.
\newblock Rankability-enhanced revenue uplift modeling framework for online marketing.
\newblock In \emph{Proceedings of the 30th ACM SIGKDD Conference on Knowledge Discovery and Data Mining}, pages 5093--5104, 2024.

\bibitem[Betlei et~al.(2024)Betlei, Vladimirova, Sebbar, Urien, Rahier, and Heymann]{betlei2024maximizing}
Artem Betlei, Mariia Vladimirova, Mehdi Sebbar, Nicolas Urien, Thibaud Rahier, and Benjamin Heymann.
\newblock Maximizing the success probability of policy allocations in online systems.
\newblock In \emph{Proceedings of the AAAI Conference on Artificial Intelligence}, volume~38, pages 11061--11068, 2024.

\bibitem[Ai et~al.(2022)Ai, Li, Gong, Yu, Xue, Zhang, Zhang, and Jiang]{ai2022lbcf}
Meng Ai, Biao Li, Heyang Gong, Qingwei Yu, Shengjie Xue, Yuan Zhang, Yunzhou Zhang, and Peng Jiang.
\newblock Lbcf: A large-scale budget-constrained causal forest algorithm.
\newblock In \emph{Proceedings of the ACM Web Conference 2022}, pages 2310--2319, 2022.

\bibitem[Chen et~al.(2024)Chen, Chen, Dowlatabadi, Hong, Iyer, Mantripragada, Narang, Pandey, Qin, Sheikh, et~al.]{chen2024practical}
Bobby Chen, Siyu Chen, Jason Dowlatabadi, Yu~Xuan Hong, Vinayak Iyer, Uday Mantripragada, Rishabh Narang, Apoorv Pandey, Zijun Qin, Abrar Sheikh, et~al.
\newblock Practical marketplace optimization at uber using causally-informed machine learning.
\newblock \emph{arXiv preprint arXiv:2407.19078}, 2024.

\bibitem[Shi et~al.(2019)Shi, Blei, and Veitch]{shi2019adapting}
Claudia Shi, David Blei, and Victor Veitch.
\newblock Adapting neural networks for the estimation of treatment effects.
\newblock \emph{Advances in neural information processing systems}, 32, 2019.

\bibitem[Nie et~al.(2021)Nie, Ye, Liu, and Nicolae]{nie2021vcnet}
Lizhen Nie, Mao Ye, Qiang Liu, and Dan Nicolae.
\newblock Vcnet and functional targeted regularization for learning causal effects of continuous treatments.
\newblock \emph{arXiv preprint arXiv:2103.07861}, 2021.

\bibitem[Kamran et~al.(2024)Kamran, Makar, and Wiens]{kamran2024learning}
Fahad Kamran, Maggie Makar, and Jenna Wiens.
\newblock Learning to rank for optimal treatment allocation under resource constraints.
\newblock In \emph{International Conference on Artificial Intelligence and Statistics}, pages 3727--3735. PMLR, 2024.

\bibitem[Ai et~al.(2024)Ai, Chen, Wang, Shang, Tao, and Li]{ai2024improve}
Meng Ai, Zhuo Chen, Jibin Wang, Jing Shang, Tao Tao, and Zhen Li.
\newblock Improve roi with causal learning and conformal prediction.
\newblock In \emph{2024 IEEE 40th International Conference on Data Engineering (ICDE)}, pages 598--610. IEEE, 2024.

\bibitem[Wang et~al.(2023)Wang, Shi, Xu, Wang, Fan, Feng, You, and Chen]{wang2023multi}
Chao Wang, Xiaowei Shi, Shuai Xu, Zhe Wang, Zhiqiang Fan, Yan Feng, An~You, and Yu~Chen.
\newblock A multi-stage framework for online bonus allocation based on constrained user intent detection.
\newblock In \emph{Proceedings of the 29th ACM SIGKDD Conference on Knowledge Discovery and Data Mining}, pages 5028--5038, 2023.

\bibitem[Mandi and Guns(2010)]{mandi2010interior}
Jayanta Mandi and Tias Guns.
\newblock Interior point solving for lp-based prediction+ optimisation, 2020.
\newblock \emph{URL http://arxiv. org/abs}, 2010.

\bibitem[Niepert et~al.(2021)Niepert, Minervini, and Franceschi]{niepert2021implicit}
Mathias Niepert, Pasquale Minervini, and Luca Franceschi.
\newblock Implicit mle: backpropagating through discrete exponential family distributions.
\newblock \emph{Advances in Neural Information Processing Systems}, 34:\penalty0 14567--14579, 2021.

\bibitem[Berthet et~al.(2020)Berthet, Blondel, Teboul, Cuturi, Vert, and Bach]{berthet2020learning}
Quentin Berthet, Mathieu Blondel, Olivier Teboul, Marco Cuturi, Jean-Philippe Vert, and Francis Bach.
\newblock Learning with differentiable pertubed optimizers.
\newblock \emph{Advances in neural information processing systems}, 33:\penalty0 9508--9519, 2020.

\bibitem[Vlastelica et~al.(2019)Vlastelica, Paulus, Musil, Martius, and Rol{\'\i}nek]{vlastelica2019differentiation}
Marin Vlastelica, Anselm Paulus, V{\'\i}t Musil, Georg Martius, and Michal Rol{\'\i}nek.
\newblock Differentiation of blackbox combinatorial solvers.
\newblock \emph{arXiv preprint arXiv:1912.02175}, 2019.

\bibitem[Shah et~al.(2022)Shah, Wang, Wilder, Perrault, and Tambe]{shah2022decision}
Sanket Shah, Kai Wang, Bryan Wilder, Andrew Perrault, and Milind Tambe.
\newblock Decision-focused learning without decision-making: Learning locally optimized decision losses.
\newblock \emph{Advances in Neural Information Processing Systems}, 35:\penalty0 1320--1332, 2022.

\bibitem[Donti et~al.(2017)Donti, Amos, and Kolter]{donti2017task}
Priya Donti, Brandon Amos, and J~Zico Kolter.
\newblock Task-based end-to-end model learning in stochastic optimization.
\newblock \emph{Advances in neural information processing systems}, 30, 2017.

\bibitem[Zhao et~al.(2017)Zhao, Fang, and Simchi-Levi]{zhao2017uplift}
Yan Zhao, Xiao Fang, and David Simchi-Levi.
\newblock Uplift modeling with multiple treatments and general response types.
\newblock In \emph{Proceedings of the 2017 SIAM International Conference on Data Mining}, pages 588--596. SIAM, 2017.

\bibitem[Hoffman et~al.(2019)Hoffman, Roberts, and Yaida]{hoffman2019robust}
Judy Hoffman, Daniel~A Roberts, and Sho Yaida.
\newblock Robust learning with jacobian regularization.
\newblock \emph{arXiv preprint arXiv:1908.02729}, 2019.

\bibitem[Ioffe and Szegedy(2015)]{ioffe2015batch}
Sergey Ioffe and Christian Szegedy.
\newblock Batch normalization: Accelerating deep network training by reducing internal covariate shift.
\newblock In \emph{International conference on machine learning}, pages 448--456. pmlr, 2015.

\end{thebibliography}

\appendix
\renewcommand{\thesection}{Appendix \Alph{section}}

\section{Theoretical Proof of Hidden Representation Clustering}

\begin{theorem}
When the Jacobian matrix of the latent variable mapping $g(x)$ satisfies the low sensitivity condition, the hidden representation $z$ is more robust to the $x$ noise than the output $y$, i.e.,  $E||\tilde{z} - z||_F^2 \ll E||\tilde{y} - y||_F^2$ when $ \tilde{x} = x + \Delta$.
\end{theorem}

\begin{proof}
Suppose that $x$ are disturbed by Gaussian noise $\Delta \sim \mathcal{N}(0, \sigma^2_{\Delta}\mathcal{I})$, and $\tilde{x} = x + \Delta$. The hidden representation $z = g(x) \in \mathcal{R}^{d_z}$, and the output $y = f(z) = f(g(x)) \in \mathcal{R}^{d_y}$, where $g(x), f(z)$ are mapping functions modeled by neural networks.

The Taylor first-order expansion of $z$ is represented by:
\begin{equation}
    \tilde{z} = g(\tilde{x}) \approx g(x) + J_g(x)\Delta
\end{equation}
where $J_g(x) \in \mathcal{R}^{d_z \times d_x}$ denotes the Jacobian matrix of $g(x)$. The perturbation error of $z$ is:
\begin{equation}
    \tilde{z} - z = J_g(x)\Delta
\end{equation}
Then, we can calculate the mean square error of the hidden representation $z$ as:
\begin{equation}
\begin{aligned}
    E||\tilde{z} - z||_F^2 &= Trace\{E[(\tilde{z} -z)(\tilde{z} - z)^T]\} \\
    &=Trace\{J_g(x)E(\Delta\Delta^T)J_g(x)^T\} \\
    &=Trace\{\sigma_{\Delta}^2J_g(x)J_g(x)^T\} \\ 
    &=\sigma_{\Delta}^2Trace\{J_g(x)J_g(x)^T\} \\
    &=\sigma_{\Delta}^2 ||J_g(x) ||_F^2
\end{aligned}
\end{equation}

The Taylor first-order expansion of $y$ is represented by
\begin{equation}
\tilde{y} \approx y + J_y(x) \Delta
\end{equation}
where $J_y(x) = \frac{\partial y}{\partial g(x)} J_g(x)$. The mean square error of $y$ is calculated as:
\begin{equation}
\begin{aligned}
        E||\tilde{y} - y||_F^2 &= E[||J_y(x) \Delta||_F^2] \\
        &=\sigma_{\Delta}^2 ||J_y(x)||_F^2 \\
        &=\sigma_{\Delta}^2 ||\frac{\partial y}{\partial g(x)} J_g(x)||_F^2 \\
        &\leq \sigma_{\Delta}^2 ||\frac{\partial y}{\partial g(x)}||_F^2 || J_g(x)||_F^2 \\
        &= \sigma_{\Delta}^2 ||\frac{\partial y}{\partial g(x)}||_F^2 ||J_g(x) ||_F^2 \\
\end{aligned}
\end{equation}
When the Jacobian matrix of $g(x)$ satisfies the low sensitivity condition \cite{hoffman2019robust}, we have:
\begin{equation}
    ||\frac{\partial y}{\partial g(x)}||_F^2 \gg 1
\end{equation}
Therefore $E||\tilde{z} - z||_F^2 \ll E||\tilde{y} - y||_F^2$, which concludes the proof. We ensure that $g(x)$ meets the low sensitivity condition by adding weight regularization and batch normalization to the hidden representation module \cite{ioffe2015batch}.
\end{proof}

\begin{theorem} 
For the model trained by RCT data, the hidden representations $Z$ are independent of $T$, while the output $Y$ is significantly affected by $T$, i.e., the mutual information $I(T;Z) =0$ and $I(T;Y)>0$.
\end{theorem} 

\begin{proof}
Suppose that $I(A;B)$ represents the mutual information between variables $A, B$, and $I(A; B|C)$ denotes the conditional mutual information of a given $C$. According to Chain Rule in Information Theory, the mutual information $I(T;Z)$ is defined by:
\begin{equation}
\begin{aligned}
        I(T;Z) &= I(Z; T|X) + I(Z;X;T) \\
        &\leq I(Z; T|X) + I(X;T)
\end{aligned}
\end{equation}
In RCT data, $X$ is independent to $T$, therefore $I(Z; T|X) \rightarrow 0$ and $I(X;T) \rightarrow 0$. Then, we have $I(T;Z) \leq 0$. Since the mutual information is always non-negative, $I(T;Z) = 0$ can be concluded. 

The mutual information $I(Y;T)$ is defined by:
\begin{equation}
\begin{aligned}
        I(Y;T) &= I(Y;T|Z) + I(Y;Z;T) \\
        &\geq I(Y;T|Z)
\end{aligned}
\end{equation}
As the output $y$ relies on $z,t$, we have:
\begin{equation}
    I(Y;T|Z) = E_{z,t} \{KL[p(y|z,t) || p(y|z)]\} > 0
\end{equation}
Therefore $I(Z;T)$ and $I(Y;T) > 0$ under the RCT data, which concludes the proof.
\end{proof}

\section{Experimental Details of Compared Algorithms}

The hyperparameters of the algorithms are reported in \autoref{appt1}. For HEU-CF and LANG-CF, we utilize generalized random forest (GRF) to model uplifts and a neural network model as the base model to predict the values of the controlled group. Considering the second stage that solves the allocation strategy, the implementation details of heuristic searching and Lagrangian duality are reported in \autoref{appalg1} and \autoref{appalg2}, respectively. For DFL-MER and DFL-PL, the allocation optimization follows the Lagrangian duality, and the DFL algorithm is indicated by \autoref{appalg3}.

\begin{table}[h]
	\caption{Hyper-parameters of HRC, DFL-PL, DFL-MER, HEU-CF, HEU-Slearner, LANG-CF, LANG-Slearner}
	\centering
	\begin{tabular}{cccccccc}
		\toprule
		Hyper-parameters & HRC   & DFL-PL & DFL-MER & HEU-CF & HEU-Slearner  & LANG-CF & LANG-Slearner \\
		\midrule
		Hidden units & 512 $\times$ 256 & 512 $\times$ 256 & 512 $\times$ 256 & 512 $\times$ 256 &  512 $\times$ 256 & 512 $\times$ 256 & 512 $\times$ 256 \\
		Learning rates & 6e-5 & 6e-5 & 6e-5 & 6e-5(base) & 6e-5  & 6e-5(base) & 6e-5 \\
            Training Epoch & 200 & 200 & 200 & 200 & 200& 200& 200 \\
            Temperature & / & 0.0 & 0.5 & / &  / & / & /  \\
		Decision factors & / & 0.3 & 0.3 & / &  / & / & / \\
		Batch sizes & 409600 & 409600 & 409600 & 409600 & 409600  & 409600 & 409600 \\
            Tree estimators & / &  /&  /& 40(GRF) & / & 40(GRF) & / \\
            Tree max samples & / & / & / & 0.2(GRF) & / & 0.2(GRF) & / \\
		\bottomrule
	\end{tabular}
	\label{appt1}
\end{table}

\begin{algorithm}[h]
\caption{Heuristic Searching Algorithm}
\label{appalg1}
\KwIn{Individual features $X$, treatments $T$, random data $\mathcal{D}$, proportions $\{p_j\}_{j=1}^{k}$ where $\sum p_j = 1$ and $k$ is determined by business logic}
\KwOut{Intervention allocation for each treatment $T$}

\BlankLine
\textbf{Training Phase:} \\
\For{each $(x, t_i, r_i) \in \mathcal{D}$}{
     $\hat{r}_i = f(x, t_i)$\;
    Update $\theta$ to minimize $\mathcal{L} = \text{MSE}(r_i, \hat{r}_i)$\;
}

\BlankLine
\textbf{Prediction Phase:} \\
\For{each $x$}{
     $\hat{r}_j = f(x, t_j)$ for all $t_j \in T$\;
     $\Delta r_j = \hat{r}_j - \hat{r}_0$\;
}
Sort all individuals by $\Delta r_j$ in descending order\;

\BlankLine
\textbf{Intervention Allocation:} \\
Divide individuals into $k$ groups with sizes $\{p_j \cdot N\}_{j=1}^{k}$\;
\For{each $t_j \in T $}{
    Apply $t_j$ to individuals in group $j$\;
}
\end{algorithm}

\begin{algorithm}[h]
\caption{Lagrangian Duality Algorithm}
\label{appalg2}
\KwIn{Individuals features $X$, treatments $T$, random data $\mathcal{D}$, Lagrangian multiplier $\lambda$}
\KwOut{Optimal treatment allocation for each individual $x$}

\BlankLine
\textbf{Training Phase:} \\
\For{each $(x, t_{\text{true}}, r, c) \in \mathcal{D}$}{
     $\hat{r} = f_r(x) = [\hat{r}_1, \hat{r}_2, \dots, \hat{r}_{|T|}]$\;
     $\hat{c} = f_c(x) = [\hat{c}_1, \hat{c}_2, \dots, \hat{c}_{|T|}]$\;
     $\text{mask}(T) = [m_1, m_2, \dots, m_{|T|}]$, where $m_i = 1$ if $t_i = t_{\text{true}}$, else $m_i = 0$\;
     $\mathcal{L}_r = \text{MSE}(r, \hat{r}) \cdot \text{mask}(T)$\;
     $\mathcal{L}_c = \text{MSE}(c, \hat{c}) \cdot \text{mask}(T)$\;
    Update model parameters $\theta_r$ and $\theta_c$ to minimize $\mathcal{L}_r + \mathcal{L}_c$\;
}

\BlankLine
\textbf{Prediction Phase:} \\
\For{each $x$}{
     $\hat{r} = f_r(x) = [\hat{r}_1, \hat{r}_2, \dots, \hat{r}_{|T|}]$\;
     $\hat{c} = f_c(x) = [\hat{c}_1, \hat{c}_2, \dots, \hat{c}_{|T|}]$\;
     $\hat{u} = \hat{r} - \lambda \cdot \hat{c}$\;
    Select $t^* = \arg\max_{t \in T} \hat{u}$\;
}

\BlankLine
\textbf{Output:} \\
Assign treatment $t^*$ to user $x$\;
\end{algorithm}

\begin{algorithm}[h]
\caption{DFL Algorithm}
\label{appalg3}
\KwIn{Individuals features $X$, treatments $T$, random data $\mathcal{D}$, Lagrangian multipliers $\lambda_{\text{list}}$, temperature $\tau$, weights $\theta_d$, $\theta_r$, $\theta_c$}
\KwOut{Optimal treatment allocation for each individual $x$}

\BlankLine
\textbf{Training Phase:} \\
Initialize $\mathcal{L}_d = 0$;\\
\For{each $(x, t_{\text{true}}, r, c) \in \mathcal{D}$}{
     $\hat{r} = f_r(x) = [\hat{r}_1, \hat{r}_2, \dots, \hat{r}_{|T|}]$\;
     $\hat{c} = f_c(x) = [\hat{c}_1, \hat{c}_2, \dots, \hat{c}_{|T|}]$\;
     $\text{mask}(T) = [m_1, m_2, \dots, m_{|T|}]$, where $m_i = 1$ if $t_i = t_{\text{true}}$, else $m_i = 0$\;

    \For{each $\lambda_i \in \lambda_{\text{list}}$}{
         $\hat{u}_{\lambda} = \hat{r} - \lambda_i \cdot \hat{c} $\;
         $q_{\lambda} = \frac{\exp(\hat{u}_{\lambda} / \tau)}{\sum \exp(\hat{u}_{\lambda} / \tau)}$\;
         $u_{\lambda} = r - \lambda_i \cdot c$\;
        Update $\mathcal{L}_d\quad += \text{mean}(q_{\lambda} \cdot \text{mask}(T) \cdot u_{\lambda}) \cdot |T|^2$\;
    }

     $\mathcal{L}_r = \text{MSE}(r, \hat{r}) \cdot \text{mask}(T)) \cdot |T|$\;
     $\mathcal{L}_c = \text{MSE}(c, \hat{c}) \cdot \text{mask}(T)) \cdot |T|$\;
     $\mathcal{L} = -\theta_d \cdot \mathcal{L}_d + \theta_r \cdot \mathcal{L}_r + \theta_c \cdot \mathcal{L}_c$\;
    Update model parameters to minimize $\mathcal{L}$\;
}

\BlankLine
\textbf{Prediction Phase:} \\
\For{each $x$}{
     $\hat{r} = f_r(x)$ and $\hat{c} = f_c(x)$ for all $T$\;
     $\hat{u} = \hat{r} - \lambda \cdot \hat{c}$\;
     $t^* = \arg\max(\hat{u})$\;
}

\BlankLine
\textbf{Output:} \\
Assign treatment $t^*$ to an individual $x$\;
\end{algorithm}

\end{document}